\documentclass[letterpaper]{article}
\usepackage{geometry}
\usepackage{listings}
\usepackage{amsthm,amsmath,amssymb}
\usepackage{thm-restate}
\usepackage{etoolbox}
\usepackage{nicefrac}
\usepackage{tikz}
\usepackage{complexity}
\usepackage{todonotes}

\usetikzlibrary{shapes.geometric}
\usetikzlibrary{positioning}
\tikzstyle{arg}=[draw,circle,fill=gray!15,inner sep=1pt,minimum size=.5cm]

\newtheorem{theorem}{Theorem}[section]
\newtheorem{example}[theorem]{Example}
\newtheorem{proposition}[theorem]{Proposition}

\newtheorem{definition}[theorem]{Definition}

\newtheorem{lemma}[theorem]{Lemma}

\newcommand{\ie}{i.e., }

\newcommand{\sts}{\mathit{sts}}

\newcommand{\nt}{\mathit{not}}
\newcommand{\HB}{\mathit{HB}}

\newcommand{\theory}{\mathit{Th}}

\newcommand{\cl}{\mathit{cl}}
\newcommand{\stb}{{\mathit{stb}}}

\newcommand{\cf}{\mathit{cf}}
\newcommand{\contrary}[1]{\overline{#1}}
\newcommand{\contraryempty}{\contrary{\phantom{a}}}

\title{
On the Correspondence of Non-flat Assumption-based Argumentation and Logic Programming with Negation as Failure in the Head}

\author{
Anna Rapberger$^1$,
Markus Ulbricht$^2$,
Francesca Toni$^1$\\[1ex]
$^1$Imperial College London, Department of Computing\\
$^2$Leipzig University, Department of Computer Science\\[1ex]
 \{a.rapberger,f.toni\}@imperial.ac.uk,\\ mulbricht@informatik.uni-leipzig.de
}
\date{}
\begin{document}

\maketitle

\begin{abstract}
The relation between (a fragment of) assumption-based argumentation (ABA) and logic programs (LPs) under stable model semantics is well-studied. 
However, for obtaining this relation, the ABA framework needs to be restricted to being flat, \ie a fragment where the (defeasible) 
assumptions can never be entailed, only assumed to be true or false. 
Here, we remove this restriction and show a correspondence between non-flat ABA and LPs with negation as failure in their head. 
We then extend this result to so-called set-stable ABA semantics, originally defined for the fragment of non-flat ABA called bipolar ABA. 
We showcase how to define set-stable semantics for LPs with negation as failure in their head and show the correspondence to set-stable ABA semantics. 
\end{abstract}

\section{Introduction}
Computational argumentation and logic programming constitute fundamental research areas in the field of knowledge representation and reasoning.
The correspondence between both research areas has been investigated 
extensively, revealing that
the computational argumentation and logic programming paradigms are inextricably linked and  provide orthogonal views on non-monotonic reasoning.
In recent years, researchers developed and studied various translations between \emph{logic programs (LPs)} and several argumentation formalisms, including translation from and to abstract argumentation~\cite{Dung95,WuCG09,CaminadaSAD15lp},
assumption-based argumentation~\cite{BondarenkoDKT97,SchulzT15,CaminadaS17,SaA21a,DBLP:journals/corr/abs-2308-15891},
argumentation frameworks with collective attacks~\cite{KonigRU22}, claim-augmented argumentation frameworks~\cite{Rapberger20,DvorakRW23}, and
abstract dialectical frameworks~\cite{AlcantaraSG19,Strass13a}.

The multitude of different translations sheds light on the close connection of 
negation as failure and argumentative conflicts. Apart from the theoretical insights, these translations are also practically enriching for both paradigms as they enable the application of methods developed for one of the formalisms to the other. On the one hand, translating logic programs to instances of formal argumentation has been proven useful for explaining logic programs~\cite{SchulzT16justifying}. Translations from argumentation frameworks into logic programs, on the other hand,  allows to utilise the rich toolbox for LPs, e.g., answer set programming solvers like clingo~\cite{DBLP:journals/tplp/GebserKKS19}, directly on instances of formal argumentation.

Existing translations consider \emph{normal LPs}~\cite{janhunen2004representing}, i.e., the class of 
LPs in which the head of each rule 
amounts precisely to one positive atom.
In this work, we take one step further and consider \emph{LPs with negation as failure in the head} of rule~\cite{INOUE199839}
. 
We investigate 
the relation of this more general class of LPs to 
\emph{assumption-based argumentation (ABA)}~\cite{BondarenkoDKT97}. 
This is a versatile structured argumentation formalism which models argumentative reasoning on the basis of assumptions and inference rules. ABA can be suitably deployed in multi-agent settings to support dialogues~\cite{FanT14} and supports applications in, e.g., healthcare~\cite{CyrasOKT21}, 
law~\cite{DungTH10-law} and robotics~\cite{FanLZLM16}.

Research in ABA often focuses on the so-called \emph{flat ABA} fragment, which prohibits deriving assumptions from inference rules. 
In this work, we show that generic (potentially non-flat) ABA (referred to improperly but compactly as \emph{non-flat ABA}~\cite{CyrasFST2018}) captures the more general fragment of LPs with negation as failure in the head, 
differently from all of the aforementioned argumentation formalisms. This underlines the increased and more flexible modelling capacities of the generic ABA formalism.

In this work, 
we investigate the relationship between non-flat ABA and LP with negation in the head, focusing 
on
\emph{stable}~\cite{BondarenkoDKT97} and \emph{set-stable}~\cite{Cyras0T17} semantics. 
While stable semantics is well understood, the latter has not been studied thoroughly so far. 
Set-stable semantics have been originally introduced for a restricted non-flat ABA fragment (\emph{bipolar ABA}~\cite{Cyras0T17}) only, with the goal to study the correspondence between ABA and a generalisation of abstract argumentation that allows for support between arguments (bipolar argumentation~\cite{CayrolL05a}). In this paper we adopt it for any non-flat ABA framework 
and study 
it in the context of LPs with negation as failure in the head.

In more detail, our contributions are as follows:
\begin{itemize}
    \item We show that each LP with negation as failure in the head corresponds to a 
    non-flat 
    ABA framework under stable semantics. 
    \item We identify an ABA fragment (\emph{LP-ABA}) 
    in which the correspondence to LPs with negation as failure in the head is 1-1.
    We prove that each non-flat ABA framework corresponds to an LPs with negation as failure in the head by showing that each ABA framework can be mapped into an LP-ABA framework.
    \item We introduce set-stable model semantics for LPs with negation as failure in the head. We identify the LP fragment corresponding to bipolar ABA 
    under set-stable semantics. We furthermore consider the set-stable semantics for any LPs with negation as failure in the head by appropriate adaptions of the reduct underpinning stable models~\cite{INOUE199839}.
\end{itemize}

\section{Background}
\label{sec:background}
We recall logic programs with negation as failure in the head~\cite{INOUE199839} and assumption-based argumentation~\cite{BondarenkoDKT97}.
\subsection{Logic programs with negation as failure 
in head}
A \emph{logic program 
with negation as failure (naf) in the head}~\cite{INOUE199839} (LP in short 
in the remainder of the paper) consists of a set of rules $r$ of the form 
\begin{align*}
a_0\gets & a_1,\dots,a_m,\nt\ b_{m+1},\dots,\nt\ b_{n}\\
\nt\ 
a_0 \gets & a_1,\dots,a_m,\nt\ b_{m+1},\dots,\nt\ b_{n}
\end{align*}
for $n \geq 0$, (propositional) atoms $a_i,b_i$, for $i\leq n$, and naf operator $\nt$.
We write $head(r)\!=\!
a_0$ and $head(r)\!=\!\nt\ 
a_0$, respectively,
and $body(r)\!=\! \{a_1,\dots,a_m,\nt\ b_{m+1},\dots,\nt\ b_{n}\}.$
Furthermore, we let
$body^+(r)=\{a_1,\dots,a_m\}$ denote the positive and $body^-(r)=\{b_{m+1},\dots,b_{n}\}$ denote the negative atoms occuring in the body of $r$; moreover, we let $head^-(r)=\{a_0\}$ if $ head(r)=\nt\ a_0$ and $head^-(r)=\emptyset$ otherwise  
(analogously for $head^+(r)$).

\begin{definition}
The \emph{Herbrand Base} of an LP $P$ is the set $\HB_P$ of all atoms occurring in $P$.
By $$\overline{\HB_P}=\{\nt\ p\mid p\in \HB_P\}$$ we denote the set of all naf-negated atoms in $\HB_P$.    
\end{definition}

We call an LP $P$ a 
\emph{normal program} 
if  $head^-(r)\!=\!\emptyset$ for each $r\in P$ and a
\emph{positive program} if $body^-(r)\!=\!head^-(r)\!=\!\emptyset$ for each $r\in P$. 
Given $I\subseteq \HB_P$
, 
the \emph{reduct} $P^I$ of 
$P$ 
is 
the positive program 
\begin{align*}
    P^I = \{ &head^+(r) \gets body^+(r) \mid \\
             &body^-(r)\cap I = \emptyset,\; head^-(r)\subseteq I \}.%
\end{align*}
We are ready to define stable LP semantics. 
\begin{definition}\label{def:stable model semantics}
$I\subseteq \HB_P$ is a \emph{stable model}~\cite{INOUE199839} of an LP $P$ if 
$I$ is a $\subseteq$-minimal Herbrand model of $P^I$, i.e., $I$ is a $\subseteq$-minimal set of atoms satisfying
\begin{itemize}
    \item[(a)] $p\!\in\! I$ iff there is a rule $r\!\in\! P^I$ s.t.\ $head(r)\!=\!p$ and $body(r)\!\subseteq\! I$;
    \item[(b)] there is no rule $r\in P^I$ with $head(r)=\emptyset$ and $body(r)\subseteq I$.
\end{itemize}
\end{definition}
Negation as failure 
in the head can be 
also interpreted 
in terms of (denial integrity) constraints, i.e., rules with empty head.
Let us consider the following example.
\begin{example}
    \label{ex:running example LP}
    Consider the LP $P$ given as follows.
    \begin{align*}
        ~P:~&p\gets\nt\ q &&q\gets\nt\ p &&s\gets  &&\nt\ s\gets s, \nt\ p.
    \end{align*}
    Here, $P$ models a choice between $p$ and $q$. However, as $s$ is factual and $\nt\ p$ entails $\nt\ s$ (together with the fact $s$), 
    $q$ is rendered impossible. 
    
    For the sets of rules $I_1 = \{p,s\}$ and $I_2 = \{q,s\}$ we obtain the following reducts: 
    \begin{align*}
        P^{I_1}:~&p\gets && &&s\gets  && \\
        P^{I_2}:~&&&q\gets &&s\gets  &&\emptyset\gets s
    \end{align*}
    We see that $I_1$ is a minimal Herbrand model of $P^{I_1}$, whereas $I_2$ is rendered invalid due to the rule $\emptyset \gets s$.
    Thus, 
    this rule can be seen as a denial integrity 
    constraint amounting to ruling out the atom~$s$.
\end{example}

\subsection{Assumption-based Argumentation}
We recall assumption-based argumentation (ABA)~\cite{BondarenkoDKT97}.
A \emph{deductive system} is a pair $(\mathcal{L},\mathcal{R})$, where  $\mathcal{L}$ is a formal language, i.e., a set of sentences,
and $\mathcal{R}$ is a set of inference rules over $\mathcal{L}$. A rule $r \in \mathcal{R}$ has the form
$$a_0 \leftarrow a_1,\ldots,a_n$$ for $n\geq 0$, with $a_i \in \mathcal{L}$.
We denote the head of $r$ by $head(r) = a_0$ and the (possibly empty)
body of $r$ with $body(r) = \{a_1,\ldots,a_n\}$.
\begin{definition}
	An \emph{ABA framework (ABAF)}~\cite{CyrasFST2018} is a tuple $(\mathcal{L},\!\mathcal{R},\!\mathcal{A},\!\contraryempty)$ for $(\mathcal{L},\!\mathcal{R})$ a deductive system, $\mathcal{A}\! \subseteq \!\mathcal{L}$ 
 the
\emph{assumptions}, and $\contraryempty\!\!:\!\!\mathcal{A}\!\rightarrow\! \mathcal{L}$ 
a \emph{contrary} function.
\end{definition}
	In this work, we focus on \emph{finite} ABAFs, i.e., $\mathcal{L}$, $\mathcal{R}$, $\mathcal{A}$ are finite; also, $\mathcal{L}$ is a \emph{set of atoms or naf-negated atoms}.

For a set of assumptions $S\subseteq \mathcal{A}$, we let $\contrary{S}=\{\contrary{a}\mid a\in S\}$ denote the set of all contraries of assumptions $a\in S$.

Below, we recall the fragment of \emph{bipolar ABAFs}~\cite{Cyras0T17}.
\begin{definition}
    An ABAF $(\mathcal{L},\!\mathcal{R},\!\mathcal{A},\!\contraryempty)$ is \emph{bipolar} iff for all rules $r\in \mathcal{R}$, it holds that $|body(r)|= 1$, $body(r)\subseteq \mathcal{A}$, and $head(r)\in \mathcal{A}\cap \contrary{\mathcal{A}}$.
\end{definition}

Next, we recall the crucial notion of \emph{tree-derivations}.
A sentence $s \in \mathcal{L}$ is \emph{tree-derivable} from assumptions $S \subseteq \mathcal{A}$ and rules $R \subseteq \mathcal{R}$, denoted by $S \vdash_R s$, if there is a finite rooted labeled tree $T$ s.t.\ the root is labeled with $s$; the set of labels for the leaves of $T$ is equal to $S$ or $S \cup \{\top\}$, where $\top \not\in \mathcal{L}$;
for every inner node $v$ of $T$ there is exactly one rule $r \in R$ such that $v$ is labelled with $head(r)$, and for each $a\in body(r)$ the node $v$ has a distinct child labelled with $a$; if $body(r)\!=\!\emptyset$, $v$ has a single child labelled~$\top$; for every rule in $R$ there is a node in $T$ labelled by $head(r)$.
We often write
$S \vdash_R p$ simply as $S \vdash p$.
Tree-derivations are the arguments in ABA; be use both notions interchangeably. 

For a set of assumptions $S$,
by $\theory_D(S)=\{p \in \mathcal L\mid \exists S'\subseteq S:S'\vdash
p\}$
we denote the set of all 
sentences derivable from 
(subsets of) $S$. Note that $S\subseteq \theory_D(S)$ since each 
$a\in \mathcal{A}$ is derivable from $\{a\}$ and rule-set $\emptyset$ ($\{a\}\vdash_\emptyset a$).
The \emph{closure} of $S$ is given by $\cl(S)=\theory_D(S)\cap \mathcal{A}$.

\begin{definition}
Let $D=(\mathcal{L},\mathcal{R},\mathcal{A},\contraryempty)$ be an ABAF. 
A set of assumptions $S\subseteq \mathcal A$ \emph{attacks} a set of assumptions $T\subseteq \mathcal A$
if $\contrary{a}\in\theory_D(S)$ for some $a\in T$.
An assumption-set $S$ is \emph{conflict-free} ($S\in\cf(D)$) if it does not attack itself; it is \emph{closed} if 
$
cl(S)=S$.    
\end{definition}

We recall 
stable~\cite{CyrasFST2018} and set-stable~\cite{Cyras0T17} ABA semantics (abbr.\ 
$\stb$ and $\sts$, respectively).
Note that, while set-stable semantics has been defined for bipolar ABAFs only, we generalise the semantics to arbitrary ABAFs.

\begin{definition}
	Let $D=(\mathcal{L},\mathcal{R},\mathcal{A},\contraryempty)$ be an ABAF. 
	Further, let $S \in \cf(D)$ be closed.
	\begin{itemize}
		\item $S\in \stb(D)$ if $S$ attacks each $\{x\} \subseteq \mathcal{A} \setminus S$;
        \item $S\in \sts(D)$ if $S$ attacks $\cl(\{x\})$ for each $x \in \mathcal{A} \setminus S$.
	\end{itemize}
\end{definition}

\begin{example}\label{ex:background ABA}
We consider an ABAF 
$D= (\mathcal{L},\mathcal{R},\mathcal{A},\contraryempty)$ with 
assumptions $\mathcal{A} = \{a,b,c\}$, their contraries $\contrary{a}$, $\contrary{b}$, and $\contrary{c}$, respectively, and rules
\begin{align*}
	\contrary{b}\gets c. && {b}\gets a. &&  \contrary{c}\gets a,b.
\end{align*}
The framework is \emph{non-flat} because we can derive $b$ from $a$.

In $D$, the set $\{c\}$ is set-stable:
Clearly, the assumption does not attack itself. 
It remains to show that the closure of $a$ and the closure of $b$ is attacked.
First note that $c$ attacks $b$ since $\{c\}\vdash \contrary{b}$.
Thus, $c$ attacks also the closure of $b$.
It follows that $c$ furthermore attacks the closure of $a$ since $\cl(\{a\})=\{a,b\}$. 
This shows that $\{c\}$ is set-stable.

Moreover, the set $\{a,b\}$ is stable and set-stable in $D$ because it is conflict-free and attacks the assumption $c$ via the argument $\{a,b\}\vdash \contrary{c}$.
\end{example}

\section{Stable Semantics Correspondence}\label{sec:stable}
In  this section, we
show that non-flat ABA under stable semantics  correspond to stable model semantics for logic programs with negation as failure in the head.
First, we show that each LP can be translated into a non-flat ABAF;
second, we present a translation from a restricted class of ABAFs (LP-ABA) into LPs; third, we extend the correspondence result to general ABAFs by providing a translation from general non-flat ABA into LP-ABA.
We conclude this section by discussing denial integrity constraints in non-flat ABA.

\subsection{From LPs to ABAFs}
Each LP $P$ can be interpreted as ABAF with assumptions $\nt\ p$ and contraries thereof, for each literal in the Herbrand base $\HB_P$ of $P$. 
We recall the translation from normal programs to flat ABA~\cite{BondarenkoDKT97}.
\begin{definition}\label{trans:LP to ABA}
The \emph{ABAF corresponding to an LP $P$} is $D_P=(\mathcal{L},\mathcal{R},\mathcal{A},\contraryempty)$ with
$\mathcal{L}=\HB_P\cup \overline{\HB_P}$,
$\mathcal{R}=P$,
$\mathcal{A}=\overline{\HB_P}$,
and $\contrary{\nt\ x}=x$ for each 
$\nt\ x \in \mathcal{A}$.
\end{definition}

\begin{example}
    Consider again the LP from Example~\ref{ex:running example LP}.
    \begin{align*}
        ~P:~&p\gets\nt\ q &&q\gets\nt\ p &&s\gets  &&\nt\ s\gets s, \nt\ p
    \end{align*}
    Here
    $D_P = (\mathcal L, \mathcal R, \mathcal A, \contraryempty)$ is the ABAF with 
    \begin{align*}
        \mathcal L =&\, \{ p,q,s, \nt\ p, \nt\ q, \nt\ s \}\\
    \mathcal R = &\, P\\
    \mathcal A = &\, \{\nt\ p, \nt\ q, \nt\ s \}
    \end{align*}
    and contrary function
    $\contrary{\nt\ x} = x$ for each $x\in \{p,q,s\}$. 
    Recall that $I_1 = \{p,s\}$ is a stable model of $P$. 
    Naturally, this set corresponds to the singleton assumption-set 
    $S = \{ \nt\ q \}$. 
    Indeed, since $p$ is derivable from $\{\nt\ q\}$ and $s$ is factual, it holds that 
    $\theory_{D_P}(S) = \{ \nt\ q, p, s \}$
    which suffices to see that $S\in\stb(D_P)$. 
\end{example}
Let us generalize the observations we made in this example. 
We translate a set of atoms $I$ (in $\HB_P$ for an LP $P$) into 
an assumption-set $\Delta(I)$
(in the ABAF $D_P$) by collecting all assumptions ``$\nt\ p$'' corresponding to the atoms \emph{outside} $I$; that is, we set $$\Delta(I)=\{\nt\ p\mid p\notin I\}.$$  
We will prove that $I$ is a stable model (in $P$) iff $\Delta(I)$ is a stable extension (in $D_P$).
First, we introduce a notion of reachability in logic programs that is based on the construction of arguments.
\begin{definition}
    Let $P$ be an LP. An atom $p\in \HB_P\cup \contrary{\HB_P}$ is \emph{reachable} from a set of naf literals $N\subseteq \contrary{\HB_P}$ iff there is a tree-based argument $N'\vdash p$ with $N'\subseteq N$ in the corresponding ABAF $D_P$.
\end{definition}
Note that the reachability target is defined for both positive and negative atoms; the source on the other hand is always a set of naf literals. The notion differs from reachability based on dependency graphs which is defined for positive atoms only.

Below, we prove our first main result.
\begin{restatable}{theorem}{thmsemanticscLPtoABA}\label{prop: P and DP semantics corresp}\label{thm:P to ABA}
Let $P$ be an LP and $D_P$ the ABAF corresponding to $P$. Then $I$ is a stable model of $P$ iff $\Delta(I)\in\stb(D_P)$. 
\end{restatable}
\begin{proof}
By definition, a set $I$ is stable iff it is $\subseteq$-minimal model in $P^I$ satisfying
\begin{itemize}
    \item[($a$)] $p\in I$ iff there is $ r\in P^I$ such that $head(r)=p$ and $body(r)\subseteq I$; and
    \item[($b$)] there is no $r\in P^I$ with $head(r)=\emptyset$ and $body^+(r)\subseteq I$.
\end{itemize}
By definition of $P^I$ we obtain $I$ is a stable model of $P$ iff $I$ is a $\subseteq$-minimal model of $P^I$ satisfying
\begin{itemize}
    \item[($a$)] $p\in I$ iff there is $ r\in P$ such that $head(r)=p$, $body^+(r)\subseteq I$, and $body^-(r)\cap I=\emptyset$; and
    \item[($b$)] there is no $r\in P$ with $head^-(r)\subseteq I$, $body^+(r)\subseteq I$, and $body^-(r)\cap I=\emptyset$.
\end{itemize}
Below, we show that the first item and the $\subseteq$-minimality requirement captures conflict-freeness (no naf literal in $I$ is derived) and the requirement that all other assumptions are attacked (all other naf literals outside $I$ are derived); whereas the second item ensures closure of the program.\footnote{We note that in the case of normal logic programs without negation in the head, the second condition does not apply. It is well known and has been discussed thoroughly in the literature that (a) holds iff $\Delta(I)$ is stable in $D_P$~\cite{SchulzT15,CaminadaS17}.}

First, Let $I$ be a stable model of $P$ and let $S=\Delta(I)$. We show that $S$ is stable in $D_P$, i.e., it is conflict-free, closed, and attacks all assumptions in $\mathcal{A}\setminus S$.
\begin{itemize}
    \item \underline{$S$ is conflict-free:}
    $S$ is conflict-free iff there is no $p\in \HB_P\setminus I$ such that $p$ is reachable, i.e., can be derived from $S$. If such a derivation would exist, then the assumption $\nt\ p\in S$ were attacked by $S$.
    
    Towards a contradiction, suppose there is an atom $p\in \HB_P\setminus I$ which is reachable from $S$. 
    Let $$Q=\{p\in \HB_P\setminus I\mid S\vdash p\}$$ denote the set of atoms that are reachable from $S$ but lie `outside' $I$. We order $Q$ according the height of the smallest tree-derivation.
    
    Wlog, we can assume that our chosen atom $p$ is minimal in $Q$, i.e., there is no other atom $q\in \HB_P\setminus I$ which is reachable in less steps.
    Let $S'\vdash p$ denote the smallest tree-derivation, and let $r$ denote the top-rule (the rule connecting the root $p$ with the fist level of the tree) of the derivation. 
    The rule satisfies $head(r)=p$, $body^-(r)\cap I=\emptyset$, and $body^+(r)\subseteq I$ (otherwise, there is an atom $q\notin I$ with a smaller tree-derivation, contradiction to the minimality of $p$ in $Q$).
    Consequently, we obtain that $p\in I$, contradiction to our initial assumption.
    \item \underline{$S$ attacks all other assumptions:} 
    Suppose there is an atom $p\in I$ which is not reachable from $S$. We show that $I'=I\setminus p$ is a  model of $P^I$. 
    That is, we show that $I'$ satisfies each rule in $P^I$. By assumption there is is no rule $r\in P$ such that $head(r)=p$, $body^+(r)\subseteq I'$, and $body^-(r)\cap I'=\emptyset$ (otherwise, $p$ is reachable from $S$).
    Hence $p\in I'$ iff there is $r\in  P^I$ such that $head(r)=p$ and $body^+(r)\subseteq I'$ is satisfied.
    $I'$ satisfies all constraints since, by assumption, there is no $r\in P^I$ with $head(r)=\emptyset$ and $body^+(r)\subseteq I$.
    Thus $I'$ is a model of $P^I$. 
    Consequently, $I$ cannot be a stable model, contradiction to our initial assumption.
    \item \underline{$S$ is closed:}
    Towards a contradiction, suppose that there is some $p\in I$ such that the corresponding naf literal $\nt\ p$ is reachable. 
    Let $r$ be the top-rule of the tree-derivation. It holds that $body^+(r)\subseteq I$ (otherwise, there is some $q\in \HB_P\setminus I$ which is reachable, contradiction to the first item), $body^-(r)\cap I=\emptyset$ and $head(r)=\nt\ p$.
    Consequently, item (b) from Definition~\ref{def:stable model semantics} is violated.
\end{itemize}
This concludes the proof of the first direction. We have shown that $S=\Delta(I)$ is stable in $D_P$.

Now, let $S=\Delta(I)$ be a stable extension in $D_P$. We show that $I$ is stable in $P$.
\begin{itemize}
    \item Let $p\in I$. Then we can construct an argument $S'\vdash p$, $S'\subseteq S$ in $D_P$, i.e., is reachable from $S$. 
    We show that there is a rule $r$ with $body^+(r)\subseteq I$, $body^-(r)\cap I=\emptyset$ and $head(r)=p$.
    We proceed by induction over the height of the argument, that is, the height of the tree-derivation.
 \begin{itemize}
    \item \underline{Base case:} Suppose $S'\vdash p$ has height 1. 
    Then there is $r\in P$ with $head(r)=p$, $body^+(r)=\emptyset$, and $body^-(r)\cap S=\emptyset$.

    \item 
\underline{$n\mapsto n+1$:}
Suppose now that the statement holds for all arguments of height smaller than or equal to $n$, and suppose $S'\vdash p$ has height $n+1$. Let $r$ denote the top-rule of the tree-derivation. 

We derive the statement by applying the induction hypothesis to all height-maximal sub-arguments (with claims in $body(r)$) of our fixed tree-derivation: 
Let $p'\in body(r)$. 
The sub-tree with root $p'$ is an argument of height $n$. Hence, by induction hypothesis, $\Delta(I)$ derives $p'$, i.e.,
there is $r'\in P$ with $head(r')=p'$, $body^+(r')\subseteq I$, and $body^-(r')\cap I=\emptyset$.
In case $p'$ is a positive literal, we obtain $p'\in I$ (by (a) from Definition~\ref{def:stable model semantics}); in case $p'$ is a naf literal, we obtain $p'\in \Delta(I)$ (by (b)). 
Since $p'$ was arbitrary, we obtain $body^+(r)\subseteq I$ and $body^-(r)\cap I=\emptyset$.
\end{itemize}
    \item For the other direction, suppose there is a rule $r\in P$ with $body^+(r)\subseteq I$, $body^-(r)\cap I=\emptyset$ and $head(r)=p$. We can construct arguments for all $body^+(r)\subseteq I$ and thus obtain $p\in I$.
    \item 
    Towards a contradiction, suppose there is a $r\in P$ with $body^+(r)\subseteq I$, $body^-(r)\cap I=\emptyset$ and $head(r)=\nt\ p$ for some $p\in I$.
    Then we can construct an argument for $\nt\ p$, contradiction to $S$ being closed. 
    \item It remains to show that $I$ is a $\subseteq$-minimal model of $P^I$. Since each atom $p\in I$ has an argument in $D_P$ we obtain minimality:
    Towards a contradiction, suppose there is a model $I'\subsetneq I$ of $P^I$. Let $p\in I\setminus I'$.  
    Since there is an argument deriving $p$ there is some $r\in P^I$ with $head(r)=p$ and $body(r)\subseteq I$, showing that $I'$ is not a model of $P^I$.\qedhere
\end{itemize}
 \end{proof}

\subsection{From ABAFs to LPs}
For the other direction, we define a mapping 
so
that each assumption corresponds to a naf-negated atom. 
However, we need to take into 
account that ABA is a more general formalism. Indeed, in LPs, there is a natural bijection between ordinary atoms and naf-negated ones (\ie $p$ corresponds to $\nt\ p$). 
Instead, in ABAFs, 
assumptions can have the same contrary, they can be the contraries of each other, and not every 
sentence is the contrary of an assumption in general. 
To show the correspondence (under stable semantics), we proceed in two steps: 
\begin{enumerate}
    \item We define the \emph{LP-ABA} fragment in which
    i) no assumption is a contrary, 
    ii) each assumption has a unique contrary, and 
    iii) no further 
    sentences exist, \ie each element in $\mathcal{L}$ is either an assumption or the contrary of an assumption. 
    We show that the translation from such LP-ABAFs to LPs is semantics-preserving. 
    \item We show that each ABAF (whose underpinning language is restricted to atoms and
their naf) can be transformed to an LP-ABAF whilst preserving semantics.
\end{enumerate}

\paragraph{Relating LP and LP-ABA}
Let us start by defining the LP-ABA fragment. A similar fragment for the case of normal LPs and flat ABAFs has been already considered~\cite{CaminadaS17,CyrasFST2018,LehtonenWJ21}. Here, we extend it to the more general case. 
\begin{definition}
    The \emph{LP-ABA  fragment} is the class of all ABAFs 
    $D=(\mathcal{L},\mathcal{R},\mathcal{A},\contraryempty)$ 
    where
         (1) $\mathcal A \cap \contrary{\mathcal A} = \emptyset$,
         (2) the contrary function
         $\contraryempty$ is injective, 
     and
         (3) $\mathcal L = \mathcal A \cup \contrary{\mathcal A}$.
\end{definition}
We show that each LP-ABAF 
corresponds to an LP, using a translation similar to~\cite{CaminadaS17}[Definition 11] (which is however for flat ABA).
We replace each assumption $a$ with  $\nt\ \contrary{a}$.
\newcommand{\repl}{\mathit{rep}}
For an atom $p\in \mathcal L$, we let 
 $$\repl(p)=
\begin{cases}
\nt\ \contrary{p}, & \text{ if }p\in \mathcal{A}\\
\contrary{a}, & \text{ if }p = \contrary{a}\in\contrary{\mathcal{A}}.
\end{cases}$$
Note that in the LP-ABA fragment, this case distinction is exhaustive. 
We extend the operator to ABA rules element-wise: 
$\repl(r)=\repl(head(r))\gets \{\repl(p)\mid r\in body(r)\}$.%
\begin{definition}\label{trans:stABAF to LP}
    For an LP-ABAF $D\!=\!(\mathcal{L},\mathcal{R},\mathcal{A},\contraryempty)$, 
    we define the \emph{associated LP} $P_D\!=\!\{\repl(r)\!\mid\! r\!\in\! \mathcal{R}\}$.
\end{definition}
\begin{example}
    Let $D$ be an ABAF with 
    $\mathcal{A}=\{p,q,s\}$ and 
    \begin{align*}
        \mathcal R~:~&\contrary{p}\gets q &&\contrary{q}\gets p &&\contrary{s}\gets  &&s\gets \contrary{s}, p.
    \end{align*}
    We replace e.g. the assumption $p$ with $\nt\ \contrary{p}$ and the contrary $\contrary{p}$ is left untouched. 
    This yields the associated LP 
    \begin{align*}
        P_D~:~&\contrary{p}\gets \nt\ \contrary{q} 
        &&\contrary{q}\gets \nt\ \contrary{p} 
        &&\contrary{s}\gets  
        &&\nt\ \contrary{s}\gets \contrary{s}, \nt\ \contrary{p}.
    \end{align*}
    Striving to anticipate the relation between $D$ and $P_D$, note that 
    $S = \{q\}\in\stb(D)$. 
    Now we compute 
$\theory_D(S)\setminus \mathcal A = \{ \contrary{p}, \contrary{s} \}$
noting that it is 
a stable model of $P_D$. 
\end{example}
It can be shown that, when restricting to 
LP-ABA, the translations in Definitions~\ref{trans:LP to ABA} and~\ref{trans:stABAF to LP} are each other's inverse. 
Below, we let $$\repl(D)=(\repl(\mathcal{L}),\repl(\mathcal{R}),\repl(\mathcal{A}),\contraryempty')$$ where $\contrary{\repl(a)}=\contrary{a}$.
\begin{restatable}{lemma}{PROPinverseLPtoLPABALP}\label{prop:inverse LP to LPABALP}
For any LP $P$, it holds that $P=P_{D_P}$.
\end{restatable}
\begin{proof}
    Each naf atom $\nt\ p$ corresponds to an assumption in $P_D$ whose contrary is $p$.
    Applying the translation from Definition~\ref{trans:stABAF to LP}, we map each assumption $\nt\ p$ to the naf literal $\nt\ \contrary{\nt\ p}=\nt\ p$. Hence, we reconstruct the original LP $P$.
\end{proof}

We obtain a similar result for the other direction, 
under the assumption that each literal is the contrary of an assumption, i.e., if $\mathcal{L}=\mathcal{A}\cup \contrary{\mathcal{A}}$ as it is the case for the LP-ABA fragment.
The translations from Definition~\ref{trans:stABAF to LP} and~\ref{trans:LP to ABA} are each other's inverse modulo the simple assumption renaming operator $\repl$ as defined above.
Note that we associate each assumption $a\in\mathcal{A}$ with $\nt\ \contrary{a}$. 

\begin{restatable}{lemma}{PROPinverseABAtoABALPABA}\label{prop:inverse ABA to ABALPABA}
    Let $D=(\mathcal{L},\mathcal{R},\mathcal{A},\contraryempty)$ be an ABAF in the LP fragment. It holds that $D_{P_D}=\repl(D)$.
\end{restatable}
\begin{proof}
    When applying the translation from ABA to LP ABA, we associate each assumption $a\in\mathcal{A}$ with a naf literal $\nt\ \contrary{a}$. 
    Applying the translation from Definition~\ref{trans:LP to ABA}, each naf literal $\nt\ \contrary{a}$ is an assumption in $D_{P_D}$. We obtain $D_{P_D}=(\repl(\mathcal{L}),\repl(\mathcal{R}),\repl(\mathcal{A}),\contraryempty')$ where $\contrary{\repl(a)}=\contrary{a}$.
\end{proof}

We are ready to prove the main result of this section.
We make use of Theorem~\ref{prop: P and DP semantics corresp} and obtain the following
result.

\begin{restatable}{theorem}{propSemanticsCorrespondencestABAtoLP}\label{prop:semantics corresp stABA LP}
\label{thm:stABA to LP}
    Let $D=(\mathcal{L},\mathcal{R},\mathcal{A},\contraryempty)$ be an LP-ABAF 
    and let $P_D$ be the associated LP
    .
    Then, $S\in\stb(D)$ iff $\theory_D(S)\setminus \mathcal{A}$ is a stable model of $P_D$.
\end{restatable}
\begin{proof}

It holds that
$S$ is stable in $D$ 
iff $$\repl(S)=\{\nt\ \contrary{a}\mid a\in S\}$$ is stable in $\repl(D)$. This in turn is equivalent to
$\repl(S)$ is stable in $D_{P_{D}}$ (by Proposition~\ref{prop:inverse ABA to ABALPABA}).
Equivalently, 
$$\{\contrary{a}\mid \nt\ \contrary{a}\notin \repl(S)\}=\{\contrary{a}\mid a\notin S\}=\theory_{D}(S)\setminus \mathcal{A}$$ is stable in $P_{D}$ (by Proposition~\ref{prop: P and DP semantics corresp}). 
This in turn holds iff $\theory_{D}(S)\setminus \mathcal{A}$ is stable in $P_D$ (by definition, $P_D=\{\repl(r)\mid r\in \mathcal{R}\}=P_{D}$).
\end{proof}

\paragraph{From ABA to LP-ABA}
To complete the correspondence result between ABA and LP, it remains to show that each ABAF $D$ can be mapped to an LP-ABAF $D'$.
To do so, we proceed as follows: 
\begin{enumerate}
    \item For each assumption $a\in \mathcal A$ we introduce a fresh atom $c_a$; in the novel ABAF $D'$, $c_a$ is the contrary of $a$. 
    \item If $p$ is the contrary of $a$ in the original ABAF $D$, then we add a rule $c_a\gets p$ to $D'$. 
    \item For any atom $p$ that is neither an assumption nor a contrary in $D$, we add a fresh assumption $a_p$ and let $p$ be the contrary of $a_p$. 
\end{enumerate}
\begin{example}\label{ex:aba to lp aba}
    Consider the ABAF $D$ with literals
    $\allowdisplaybreaks \mathcal L = \{ a, b, c, p, q \}$,
    assumptions
    $\mathcal A = \{ a, b, c \}$, 
    and their contraries
    $\contrary{a} = p$,
    $\contrary{b} = p$, 
    and 
    $\contrary{c} = a$, respectively, with rules 
    \begin{align*}
        ~\mathcal R: ~ & r_1 = p\gets a,b && r_2 = q\gets a,b && r_3 = p \gets c.
    \end{align*}
    First note that $\{c\}\in\stb(D)$. 
    We construct the LP-ABAF $D'$
    by adding rules $c_a \gets p$, $c_b \gets p$, and $c_c \gets a$; $c_a$, $c_b$, and $c_c$ are the novel contraries. 
    Moreover, $q$ is neither a contrary nor an assumption, so we add a novel assumption $a_q$ with contrary $q$. 
    The stable extension $\{c\}$ is only preserved under projection: we now have $\{c, a_q\}\in\stb(D')$. 
\end{example}
We show that each ABAF $D$ can be mapped into an (under projection) equivalent ABAF $D'$. We furthermore note that the translation can be computed efficiently.
\begin{restatable}{proposition}{propMAPABAtostABA}\label{prop:ABAstABAsemanticscorresp}
    For each ABAF
    $D=(\mathcal{L},\mathcal{R},\mathcal{A},\contraryempty)$ there is ABAF $D'$ computable in polynomial time s.t.\ 
    (i) $D'$ is an LP-ABAF and 
    (ii) $S\in\stb(D')$ iff $S\cap \mathcal A\in\stb(D)$. 
\end{restatable}
\begin{proof}
Let $D=(\mathcal{L},\mathcal{R},\mathcal{A},\contraryempty)$ be an ABAF and let $D'=(\mathcal{L}',\mathcal{R}',\mathcal{A},\contraryempty')$ be ABAF constructed as described, \ie 
\begin{enumerate}
    \item For each assumption $a\in \mathcal A$ we introduce a fresh atom $c_a$; in the novel ABAF $D'$, $c_a$ is the contrary of $a$. 
    \item If $p$ is the contrary of $a$ in the original ABAF $D$, then we add a rule $c_a\gets p$. 
    \item For any atom $p$ that is neither an assumption nor a contrary in $D$, we add a fresh assumption $a_p$ and let $p$ be the contrary of $a_p$ in $D'$. 
\end{enumerate}
First of all, the construction is polynomial. 
Towards the semantics, let us denote the result of applying steps (1) and (2) by $D^*$. 
We show that in $D$ and $D^*$ the attack relation between semantics persists. 

Let $S\subseteq \mathcal A$ be a set of assumptions. In the following, we make implcit use of the fact that entailment in $D$ and $D^*$ coincide except the additional rules deriving certain contraries in $D^*$. 

($\Rightarrow$)
Suppose $S$ attacks $a$ in $D$ for some $a\in\mathcal A$. Then 
$p\in\theory_D(S)$ where $p=\contrary{a}$. 
By construction, $p\in\theory_{D^*}(S)$ as well and since $p = \contrary{a}$, the additional rule $c_a\gets p$ is applicable. Consequently, $c_a\in\theory_{D^*}(S)$, \ie  $S$ attacks $a$ in $D^*$ as well. 

($\Leftarrow$)
Now suppose $S$ attacks $a$ in $D^*$ for some $a\in\mathcal A$. 
Then $c_a\in\theory_{D^*}(S)$ which is only possible whenever $p\in\theory_{D^*}(S)$ holds for $p$ the original contrary of $a$. Thus $S$ attacks $a$ in $D$. 

We deduce $$\stb(D) = \stb(D^*).$$

Finally, for moving from $D^*$ to $D'$ we note that adding assumptions $a_p$ (which do not occur in any rule) corresponds to adding arguments without outgoing attacks to the constructed AF $F_{D^*}$. This has (under projection) no influence on the stable extensions of $D^*$. 
Consequently 
\begin{align*}
    E\in\stb(D') 
    \; \Leftrightarrow
    \; E\cap \mathcal A\in\stb(D^*)
    \; \Leftrightarrow
    \; E\cap \mathcal A\in\stb(D).
\end{align*}
as desired. 
\end{proof}
Given an ABAF $D$, we combine the previous translation with Definition~\ref{trans:stABAF to LP} to obtain the associated LP $P_D$. 
Thus, each ABAF $D$ can be translated into an LP, as desired.
\begin{example}
  Let us consider again the ABAF $D$ from Example~\ref{ex:aba to lp aba}.
  As outlined before, applying the translation into an LP-ABA $D'$ yields an ABAF $D'$ with assumptions 
  $\mathcal A = \{ a, b, c,  a_q, a_p\}$
  their contraries 
    $\contrary{a} = c_a$,
    $\contrary{b} = c_b$,  
    $\contrary{c} = c_c$,  
    $\contrary{a_q}= q$, and
    $\contrary{a_p}= p$, 
    respectively, and with rules 
    \begin{align*}
         p\gets a,b. && q\gets a,b. &&  p \gets c. \\
         c_a \gets p. && c_b \gets p. && c_c \gets a.
    \end{align*}
    The resulting framework lies in the LP-ABA class. 
    In the next step, we apply the translation from LP-ABA to LP and obtain the associated LP $P_D$ with rules
    \begin{align*}
         p\gets \nt\ c_a, \nt\ c_b. && q\gets \nt\ c_a, \nt\ c_b. &&  p \gets \nt\ c_c. \\
         c_a \gets p. && c_b \gets p. && c_c \gets \nt\ c_a.
    \end{align*} 
    The set $\{p,c_a,c_b\}$ is the stable model corresponding to our stable extension $\{c\}$ from $D$ (under projection).
\end{example}

\subsection{Denial Integrity Constraints in ABA}
Our correspondence results allow for a novel interpretation of the derivation of assumptions in ABA in the context of stable semantics.
Analogous to the correspondence of naf in the head and allowing for constraints (rules with empty head) in LP we can view the derivation of an assumption as \emph{setting constraints}: 
for a set of assumptions $M\subseteq \mathcal{A}$ and an assumption $a\in \mathcal{A}$,
a derivation $M\vdash a$ intuitively captures the constraint $\gets M,\contrary{a}$, i.e., one of $M\cup \{\contrary{a}\}$ is false. 

Thus, our results indicate that deriving assumptions is the same  as imposing constraints.
More formally, the following observation can be made. 

\begin{restatable}{proposition}{propNonFlatRulesConstraint}
    \label{prop:non-flat rules contraint}
    Let 
    $D\!=\!(\mathcal{L},\mathcal{R},\mathcal{A},\contraryempty)$ be an ABAF 
    and let 
    $D'\!=\!(\mathcal{L},\mathcal{R}\cup \{r\},\mathcal{A},\contraryempty)$
    for a rule $r$ of the form 
    $a\!\gets \! M$ 
    with $M\cup \{a\}\!\subseteq \mathcal A$. 
    Then, $S\!\in\!\stb(D')$ iff 
    (i) $S\!\in\!\stb(D)$ and 
    (ii) $M\!\not\subseteq \!S$ or $a\!\in \! S$. 
\end{restatable}
\begin{proof}
    We first make the following observation. 
    We have 
    $$\forall S\subseteq \mathcal A: \theory_D(S)\subseteq \theory_{D'}(S)$$
    by definition and 
    $$
    p\in\theory_{D'}(S)\setminus \theory_{D}(S) \Rightarrow a\notin S
    $$
    because the only additional way to make deriviations in $D'$ is through a rule entailing $a$. 
    This, however, implies 
    \begin{align}
        \label{eq:closed sets th}
        S \text{ closed in }D' \Rightarrow \theory_D(S) = \theory_{D'}(S),       
    \end{align}
    \ie for sets closed in $D'$, the derived atoms coincide. 

    Now let us show the equivalence. 
    
    ($\Rightarrow$) 
    Suppose $S\in\stb(D')$. 
    Since $S$ is closed, $M\not\subseteq S$ or $a\notin S$, so condition (ii) is met. 
    Moreover, by \eqref{eq:closed sets th}, $S$ is conflict-free and attacks each $a\notin S$ in $D$, \ie $D\in\stb(D)$. Thus condition (i) is also met. 
    
    ($\Leftarrow$) 
    Let $S\in\stb(D)$ and let $M\not\subseteq S$ or $a\in S$. Then $S$ is also closed in $D'$. We apply \eqref{eq:closed sets th} and find $S\in\stb(D')$. 
\end{proof}

\begin{example}
    Consider the ABAF $D$ with assumptions
    $\mathcal A = \{ a, b, c, d \}$, 
    and their contraries
    $\contrary{a}$,
    $\contrary{b}$,
    $\contrary{c}$, and
    $\contrary{d}$,
    respectively, with rules 
    \begin{align*} 
        r_1=\contrary{c} \gets a,b. && r_2=\contrary{a} \gets c.
    \end{align*}
    The ABAF $D$ has two stable models: $S_1=\{a,b,d\}$ and $S_2=\{b,d,c\}$.
    
    Consider the ABAF $D'$ where we add a new rule $$r_3=a\gets d.$$ 
    Intuitively, this rule encodes the \emph{constraint} $\gets \contrary{a},d$, i.e., 
    $\contrary{a}$ and $d$ cannot be true both at the same time.
    Consequently, the ABAF $D'$ has a single stable model $S_1$. 
\end{example}

\section{Set-Stable Model
Semantics}
\label{sec:setstable}
In this section, we investigate set-stable semantics in the context of logic programs.

Set-stable semantics has been originally introduced for bipolar ABAFs (where each rule is of the form $p\gets a$ with $a$ an assumption and $p$ either an assumption or the contrary thereof) for capturing existing notions of stable extensions for bipolar (abstract) argumentation; we will thus first identify the corresponding LP fragment of \emph{bipolar LPs} and introduce the novel semantics therefor. 
We then show that this semantics corresponds to set-stable ABA semantics, even in the general case.
Interestingly, despite being the formally correct counter-part to set-stable ABA semantics, 
the novel LP semantics exhibits non-intuitive behavior in the general case, as we will discuss. 

\newcommand{\supp}{\textit{supp}}

\subsection{Bipolar LPs and Set-Stable Semantics}  
Recall that an ABAF $D = (\mathcal L, \mathcal R, \mathcal A, \contraryempty)$ is \emph{bipolar} iff each rule is of the form $p\gets a$ where $a$ is an assumption and $p$ is either an assumption or the contrary of an assumption. 
We adapt this to LPs as follows.%
\begin{definition}
    The \emph{bipolar LP fragment} is the class of LPs $P$ with $|body(r)|\!=\!1$ and $body(r)\!\subseteq\!\contrary{\HB_P}$ for all $r\in P$.
\end{definition}
In ABA, set-stable semantics relax stable semantics: it 
suffices 
if the \emph{closure} of an assumption $a$ 
outside a given set is attacked; that is, it suffices if $a$ ``supports'' an attacked assumption $b$, e.g., if the ABAF contains the rule $b\gets a$. 
Let us discuss this 
for bipolar LPs:
given a set of atoms 
$I\!\subseteq \!\HB_P$ in a program $P$, 
we can accept an atom $p$ not only if it is reachable from $\Delta(I)$,
but also if there is some reachable $q$ and $\nt\ p$ ``supports'' $\nt\ q$.
For instance, given the rule of the form 
$\nt\ q \!\gets \!\nt\ p \!\in \! P$, 
 we are allowed to add the \emph{contraposition} $p\gets q$ to the program 
 $P$ before evaluating our potential model $I$.

To capture all ``supports'' between naf-negated atoms
, 
we define their \emph{closure}
, amounting to 
the set of all positive and naf-negated atoms obtainable by forward chaining. 
\begin{definition}
    For a bipolar LP $P$ and a set $S\!\subseteq\! \HB_P\!\cup\! \overline{\HB_P}$, we define 
    $$\supp(S)=S\cup \{l\mid \exists r\in P:body(r)\subseteq S, head(r)=l\}.$$ 
    The \emph{closure} of $S$ is defined as $\cl(S)=\bigcup_{i>0}\supp^i(S)$.
\end{definition}
Note that $\cl(S)$ returns positive as well as negative atoms. For a singleton $\{a\}$, we write $\cl(a)$ instead of $\cl(\{a\})$. 
\begin{example}
    \label{ex:set-stable bipolar closure}
    Consider the bipolar LP $P$ given as follows.
    \begin{align*}
        ~P:~&p\gets\nt\ p &&\nt\ q\gets\nt\ p &&q\gets  \nt\ s.
    \end{align*}
    Then, $\cl(\{\nt\ p\})=\{p,\nt\ q, \nt\ p\}$, $\cl(\{\nt\ q\})=\{\nt\ q\}$, and $\cl(\{\nt\ s\})=\{q, \nt\ s\}$.
\end{example}
We define a modified reduct 
by adding rules to make the closure explicit:
for each atom $a\in \HB_P$, if $\nt\ b$ can be reached from 
$\nt\ a$, we add the rule $a\gets b$. 
\begin{definition}\label{def:set-stabe reduct}
    For a bipolar LP $P$ and 
    $I\!\subseteq \! \HB_P$, 
    the \emph{set-stable reduct} $P^I_s$ of $P$ is defined as $P^I_s=P^I\cup P_s$ where $$P_s= \{a\gets b\mid a,b\!\in \! \HB_P, a\!\neq \! b,\nt\ b\in \cl(\{\nt\ a\})\}.$$
\end{definition}
Note that we require $a\!\neq \! b$ to avoid constructing redundant rules of the form ``$a\gets a$''. 

\begin{example}
Let us consider again the LP $P$ from Example~\ref{ex:set-stable bipolar closure}.
Let $I_1=\{q\}$ and $I_2=\{p,q\}$.
We compute the set-stable reducts according to Definition~\ref{def:set-stabe reduct}. First, we compute the reducts $P^{I_1}$ and $P^{I_2}$.
Second, for each naf literal $\nt\ x$, we add a rule $x\gets y$, for each $y\in \HB_P$ with $\nt\ y\in \cl(\{\nt\ x\})$, to both reducts.
Inspecting the computed closures of the naf literals of $P$, this amounts to adding the rule $(p\gets q)$ to each reduct. 

Overall, we obtain 
    \begin{align*}
        P^{I_1}_s:~&p\gets  &&\emptyset\gets &&q\gets && p\gets q \\
        P^{I_2}_s:~& && &&q\gets && p\gets q 
    \end{align*}
\end{example}
We are ready to give the definition of set-stable semantics. Note that we state the definition for arbitrary (not only bipolar) LPs.
\begin{definition}\label{def:set stable model sem}
An interpretation 
$I\subseteq \HB_P$ is a \emph{set-stable model} of an LP $P$ if 
$I$ is a $\subseteq$-minimal model of $P^I_s$ 
satisfying
\begin{itemize}
    \item[(a)] $p\!\in\! I$ iff there is $r\!\in\! P^I_s$ s.t.\ $head(r)\!=\!p$ and $body(r)\!\subseteq\! I$;
    \item[(b)] there is no rule $r\in P^I_s$ with $head(r)=\emptyset$ and $body(r) \subseteq I$.
\end{itemize}
\end{definition}
\begin{example}
    \label{ex:set-stable reduct}
    Consider again the LP 
    $P$ from Example~\ref{ex:set-stable bipolar closure}. 
    It can be checked that $P$ has no stable model. Indeed, the reduct $P^{I_1}$ contains the unsatisfiable rule $(\emptyset\gets)$; the set $I_2=\{p,q\}$ on the other hand is not minimal for $P^{I_2}$.
    
    If we consider the generalised set-stable reduct instead, we find that the set $I_2$ is a $\subseteq$-minimal model for $P^{I_2}_s$. 
    The atom $q$ is factual in $P^{I_2}_s$ and the atom $p$ is derived by $q$.
    Thus, $I_2$ is set-stable in $P$.
\end{example}

\subsection{Set-stable Semantics in general (non-bipolar) LPs}
So far, we considered set-stable model semantics in the bipolar LP fragment. 
As it is the case for the set-stable ABA semantics, our definition of set-stable LP semantics generalises to arbitrary LPs, 
beyond the bipolar class.

Set-stable model semantics belong to the class of two-valued semantics, that is, each atom is either set to true or false (no undefined atoms exist). Moreover, set-stable model semantics generalise stable model semantics: each stable model of an LP 
is set-stable, but not vice versa, as Example~\ref{ex:set-stable reduct} shows.%
\begin{proposition}
Let $P$ be an LP.
    Each stable model $I$ of $P$ is set-stable (but not vice versa).  
\end{proposition}
\begin{proof}
    Let $I$ denote a stable model of $P$.
    By definition, the generalised reduct $P^I_s$ of $P^I$ is a superset of all rules in $P^I$.
    Thus (a) and (b) in Definition~\ref{def:set stable model sem} are satisfied. Moreover, $I$ is $\subseteq$-minimal by Definition~\ref{def:stable model semantics}.
\end{proof}
We furthermore note that the support of a set of positive and negative atoms can be computed in polynomial time. 
\begin{lemma}
    \label{lem:computing the closure is in P}
     For a bipolar LP $P$ and a set $S\subseteq \HB_P\cup \overline{\HB_P}$, the computation of $\cl(S)$ is in $\P$. 
\end{lemma}
It follows that the computation of a set-stable model of a given program $P$ is of the same complexity as finding a stable model.

In the case of general LPs, however, the novel semantics exhibits counter-intuitive behavior, as the following example demonstrates.%
\begin{example}\label{ex:unintuitive set-stable}
    Consider the following two LPs $P_1$ and $P_2$:
        \begin{align*}
        P_1:~&q\gets  &&\nt\ q\gets \nt\ p \\
        P_2:~&q\gets  &&\nt\ q\gets \nt\ p, \nt\ s.
    \end{align*}
    In $P_1$ the set $\{p,q\}$ is set-stable because we can take the contraposition of the rule and obtain $p\gets q$. This is, however, not possible in $P_2$ which in fact has no set-stable model.
\end{example}
The example indicates that the semantics does not generalise well to arbitrary LPs.
We note that a possible and arguably intuitive generalisation of set-stable model semantics would be to allow for contraposition for all rules that derive a naf literal. This, however, requires disjunction in the head of rules. 
Applying this idea to Example~\ref{ex:unintuitive set-stable} yields the rule $p \vee s \gets q$  when constructing the reduct with respect to $P_2$. The resulting instance therefore lies in the class of disjunctive LPs (a thorough investigation of this proposal however is beyond the scope of the present paper).

\subsection{Relating  ABA and LP under set-stable semantics}
In the previous subsection, we identified certain shortcomings of set-stable semantics when applied to general LPs. 
This poses the question whether our formulation of set-stable LP semantics is indeed the LP-counterpart of set-stable ABA semantics. 
In this subsection, we show that, despite the unwanted behavior of set-stable model semantics for LPs, the choice of our definitions is correct: 
set-stable ABA and LP semantics correspond to each other.
We show that our novel LP semantics indeed captures the spirit of ABA set-stable semantics, even in the general case. 

We show that the semantics correspondence is preserved under the translation presented in Definition~\ref{trans:LP to ABA}. 
We prove the following theorem.
\begin{theorem}[restate=thmLPtoABA, name= ]\label{thm: LP to ABA set stab;e corresp}
   For an LP $P$ and its associated ABAF $D_P$, $I$ is set-stable in $P$ iff $\Delta(I)$ is set-stable in $D_P$.
\end{theorem}
\begin{proof}
By definition, $I$ is set-stable iff it is a $\subseteq$-minimal model of $P^I_s$ satisfying
\begin{itemize}
    \item[(a)] $p\!\in\! I$ iff there is $r\!\in\! P^I_s$ s.t.\ $head(r)\!=\!p$ and $body(r)\!\subseteq\! I$;
    \item[(b)] there is no $r\in P^I_s$ with $head(r)=\emptyset$ and $body(r)\subseteq I$.
\end{itemize}
Equivalently, by definition of $P^I_s$,
\begin{itemize}
    \item[($a$)] $p\in I$ iff 
    \begin{itemize}
        \item[(1)] there is $r\in P$ s.t.\ $head(r)=p$, $body^+(r)\subseteq I$ and $body^-(r)=\emptyset$; or
        \item[(2)] there is $q\in I$ such that $\nt\ q\in \cl(\nt\ p)$ and there is
    $r\in P$ s.t.\ $head(r)=q$, $body^+(r)\subseteq I$ and $body^-(r)=\emptyset$; and
    \end{itemize}
    \item[(b)] there is no $r\in P$ with $head^-(r)\subseteq I$, $I\subseteq body^+(r)$, and $body^-(r)=\emptyset$. 
\end{itemize}
The second item (b) is analogous to the proof of Theorem~\ref{thm:P to ABA}; item (a1) corresponds to item (a) of the proof of Theorem~\ref{thm:P to ABA}.
Item (a2) formalises that it suffices to (in terms of ABA) attack the closure of a set. 

Let $I$ be a set-stable model of $P$. We show that $S=\Delta(I)$ is set-stable in $D_P$, i.e., $S$ is conflict-free, closed, and attacks the closure of all remaining assumptions. The first two points are analogous to the proof of Theorem~\ref{thm:P to ABA}. Below we prove the last item.
\begin{itemize}
    \item \underline{$S$ attacks the closure of all other assumptions:} 
    Suppose there is an atom $p\in I$ which is not reachable from $S$ and there is no $q\in I$ with $\not q\in \cl(\not p)$. 
    Similar to the proof in Theorem~\ref{thm:P to ABA}, we can show that $I'=I\setminus p$ is a model of $P^I_s$.
     By assumption there is is no rule $r\in P$ such that $head(r)=p$, $body^+(r)\subseteq I'$, and $body^-(r)\cap I'=\emptyset$ (otherwise, $p$ is reachable from $S$); moreover, there is no rule $p\gets q$ in $P_s$ (otherwise, $\nt\ p$ is in the support from $\nt\ q$). 
    We obtain that $I'$ is a model of $P^I_s$, contradiction to our initial assumption.
\end{itemize}
Next, we prove the other direction. Let $S=\Delta(I)$ be a set-stable extension of $D_P$. We show that $I$ is set-stable in $P$. Similar to the proof of Theorem~\ref{thm:P to ABA} we can show that all constraints are satisfied and that $I$ is indeed minimal. Also, the remaining correspondence proceeds similar as in the case of stable semantics, as shown below.
\begin{itemize}
    \item Let $p\in I$. Then either we can construct an argument $S'\vdash p$, $S'\subseteq S$ in $D_P$, or there is some $q\in I$ such that $\nt\ q\in \cl(\nt\ p)$ for which we can construct an argument in $D_P$.
    If the former holds, then we proceed analogously to the corresponding part  in the proof of Theorem~\ref{thm:P to ABA} and item (a1) is satisfied.
    
    Now, suppose the latter is true. 
    Analogously to the the proof of Theorem~\ref{thm:P to ABA}, we can show that there is a rule $r\in P$ with $body^+(r)\subseteq I$, $body^-(r)\cap I=\emptyset$ and $head(r)=q$, that is (a2) is satisfied.
    
    \item For the other direction, suppose there is a rule $r\in P$ with $body^+(r)\subseteq I$, $body^-(r)\cap I=\emptyset$ and $head(r)=p$ and there is $q\in I$ with
    $\nt\ q\in \cl(\nt\ p)$ and
    $head(r)=q$, $body^+(r)\subseteq I$ and $body^-(r)=\emptyset$ for some $r$.
    We can construct arguments for all $body^+(r)\subseteq I$ and thus $p\in I$.\qedhere
\end{itemize}
\end{proof}

Analogous to the case of stable semantics, we can show that the LP-ABA fragment preserves the set-stable semantics and obtain the following result.
\begin{restatable}{theorem}{propSemanticsCorrespondencestABAtoLPsetstable}
    Let $D$ be an LP-ABAF 
    and let $P_D$ be the associated LP
    .
    Then, $S\in\sts(D)$ iff $\theory_D(S)\setminus \mathcal{A}$ is a set-stable model of $P_D$.
\end{restatable}

Making use of the translation from general ABA to the LP-ABA fragment outlined in the previous section, we obtain that the correspondence extends to general ABA.

\subsection{Set-stable Semantics for General (non-bipolar) ABAFs}

Recall that Example~\ref{ex:unintuitive set-stable} indicates that the semantics does not generalise well in the context of LPs.
In light of the close relation between ABA and LP, it might be the case that the non-intuitive behavior affects set-stable ABA semantics.
However, we find that set-stable semantics generalise well for ABAFs. The reason lies in the differences between deriving assumptions (in ABA) and naf literals (in LPs) beyond classical stable model semantics. 

Let us translate Example~\ref{ex:unintuitive set-stable} in the language of ABA.
\begin{example}\label{ex:set-stable unwanted behavior ABA}
The translation of the LPs $P_1$ and $P_2$ from Example~\ref{ex:unintuitive set-stable} yields two ABAFs $D_1$ and $D_2$.
The ABAF $D_1$ has two assumptions $\mathcal{A}_1=\{a,b\}$ (representing $\nt\ p$ and $\nt\ q$, respectively) with contraries $\contrary{a}$ and $\contrary{b}$, and rules
        \begin{align*}
        \mathcal{R}_1:~&\contrary{b}\gets.  && b \gets a. 
    \end{align*}
The ABAF $D_2$ has three assumptions $\mathcal{A}_2=\{a,b,c\}$ (representing $\nt\ p$, $\nt\ q$, and $\nt\ s$, respectively) with contraries $\contrary{a}$, $\contrary{b}$, $\contrary{c}$, and rules
        \begin{align*}
        \mathcal{R}_2:~&\contrary{b}\gets. && b \gets a, c.
    \end{align*}
    By Theorem~\ref{thm: LP to ABA set stab;e corresp}, we obtain the set-stable extensions of the ABAFs from our results from the original programs $P_1$ and $P_2$. 
    In $D_1$, the empty set is set-stable because it attacks the closure of each assumption. 
    In $D_2$, on the other hand, no set of assumptions is set-stable: $a$ and $c$ are not attacked, although they jointly derive $b$ which is attacked by the empty set. 
\end{example}
In contrast to the LP formulation of the problem where taking the contraposition of each rule with a naf literal in the head would have been a more natural solution, the application of set-stable semantics in the reformulation of Example~\ref{ex:unintuitive set-stable} confirms our intuition. 
The set $\{a,c\}$ derives the assumption $b$, however, 
the attack onto $b$ is not propagated to (the closure of) one of the members of $\{a,c\}$.

The example indicates a fundamental difference between deriving assumptions and naf literals in ABA and LPs, respectively.
A rule in an LP with a naf literal in the head is interpreted as denial integrity constraint (under stable model semantics).
As a consequence, the naf literal in the head of a rule is replaceable with any positive atom in the body; e.g., the rules $\nt\ p \gets  q, s$ and $\nt\ s \gets  q, p$ are equivalent as they both formalise the constraint $\gets p, q, s$.
Although a similar behavior of rules with assumptions in the head can be identified in the context of stable semantics in ABA, 
the derivation of an assumption goes beyond that; it indicates a hierarchical dependency between assumptions. 

\section{Discussion}
In this work, we investigated the close relation between non-flat ABA and LPs with negation as failure in the head, focusing on stable and set-stable semantics. 
Research often focuses on the flat ABA fragment in which each set of assumptions is closed.
This restriction has however certain limitations; as the present work demonstrates, non-flat ABA is capable of capturing a more general LP fragment, thus opening up more broader application opportunities.
To the best of our knowledge, our work provides the first correspondence result between an argumentation formalism and a fragment of logic programs which is strictly larger than the class of normal LPs. 
We furthermore studied set-stable semantics, originally defined only for bipolar ABAFs, in context of general non-flat ABAFs and LPs. 

The provided translations has practical as well as theoretical benefits. 
Conceptually, switching views between deriving assumptions (as possible in non-flat ABA) and imposing denial integrity constraints (as possible in many standardly considered LP fragments) allows us to look at a problem from different angles; oftentimes, it can be helpful to change viewpoints for finding solutions.
Practically, our translations yield mutual benefits for both fields.
Our translations from ABA into LP yield a solver for non-flat ABA instances (as, for instance, employed in~\cite{DBLP:journals/corr/abs-2405-11250}), as commonly used ASP solvers (like clingo~\cite{DBLP:journals/tplp/GebserKKS19}) can handle constraints. With this, we provide a powerful alternative to solvers for non-flat ABA, which are typically not supported by established ABA solvers due to the primary focus on flat instances (with some exceptions~\cite{DBLP:journals/corr/abs-2404-11431,DBLP:conf/tafa/Thimm17}).
LPs can profit from the thoroughly investigated explanation methods for ABAFs~\cite{CyrasFST2018,FanLZML17,FanT11}.

The generalisation of set-stable model semantics to the non-bipolar ABA and LP fragment furthermore indicated interesting avenues for future research. 
As Example~\ref{ex:unintuitive set-stable} indicates, the semantics  
does not generalise 
well beyond the bipolar LP fragment.
It would be interesting to further investigate reasonable generalisations for set-stable model semantics for LPs. As discussed previously, a promising generalisation might lead us into the fragment of disjunctive LPs.
Another promising direction for future work would be to further study and develop denial integrity constraints in the context of ABA, beyond stable semantics. 
A further interesting avenue for future work is the development and investigation of three-valued semantics (such as partial-stable or L-stable model semantics) for LPs with negation as failure in the head, in particular in correspondence to their anticipated ABA counter-parts (e.g., complete and semi-stable semantics, respectively). 

\section*{Acknowledgments}

This research was funded in whole, or in part, 
by the  European Research Council (ERC) under the European Union’s Horizon 2020 research and innovation programme (grant agreement No. 101020934, ADIX) and by J.P. Morgan and by the Royal Academy of Engineering under the Research Chairs and Senior Research Fellowships scheme; 
by  the  Federal  Ministry  of  Education  and  Research  of  Germany  and  by  S\"achsische Staatsministerium  f\"ur  Wissenschaft,  Kultur  und  Tourismus  in  the  programme  Center  of Excellence for AI-research ``Center for Scalable Data Analytics and Artificial Intelligence Dresden/Leipzig'', project identification number:  ScaDS.AI. 
\bibliographystyle{plain}

\end{document}